\documentclass[10pt,a4paper]{article}

\usepackage[T1]{fontenc}
\usepackage[utf8]{inputenc}
\usepackage[a4paper,top=18mm,bottom=20mm,left=17mm,right=17mm,columnsep=6mm]{geometry}
\usepackage{lmodern}
\usepackage{microtype}
\usepackage{graphicx}
\usepackage{amsmath}
\usepackage{amssymb}
\usepackage{amsthm}
\usepackage{mathtools}
\usepackage{bm}
\usepackage{xcolor}
\usepackage{tikz}
\usetikzlibrary{arrows.meta,shapes.geometric}
\usepackage{float}
\usepackage{tabularx}
\usepackage{flushend}
\usepackage{authblk}
\usepackage{orcidlink}
\usepackage{bbding}
\usepackage[font=small,labelfont=bf]{caption}
\usepackage{hyperref}
\hypersetup{colorlinks=true,linkcolor=black,citecolor=black,urlcolor=blue}
\usepackage[capitalise,nameinlink,noabbrev]{cleveref}

\newtheorem{theorem}{Theorem}[section]

\theoremstyle{remark}
\newtheorem{remark}{Remark}

\newfloat{algorithm}{t}{loa}
\floatname{algorithm}{Algorithm}
\crefname{algorithm}{Algorithm}{Algorithms}
\Crefname{algorithm}{Algorithm}{Algorithms}
\crefname{assumption}{Assumption}{Assumptions}
\Crefname{assumption}{Assumption}{Assumptions}

\newcommand{\citep}[1]{\cite{#1}}

\newcommand{\R}{\mathbb{R}}
\newcommand{\E}{\mathbb{E}}

\newcommand{\keywords}[1]{%
  \par\medskip\noindent\textbf{Keywords: }%
  {\def\and{\unskip\;\textperiodcentered\;}#1}%
}

\title{\bfseries Friction-Augmented Drifting Models for Resource-Efficient Domain Translation}
\author[1]{\mbox{Arkadii Kazanskii\,\orcidlink{0009-0004-8941-8565}}}
\author[1]{\mbox{Tatiana Petrova\,\orcidlink{0000-0001-5569-145X}}}
\author[2,3,4]{\mbox{Andrey Ustyuzhanin\,\orcidlink{0000-0001-7865-2357}}}
\author[1]{\mbox{Konstantin Bagrianskii\,\orcidlink{0009-0002-5989-163X}}}
\author[1]{\mbox{Aleksandr Puzikov\,\orcidlink{0009-0005-2598-0927}}}
\author[1]{\mbox{Radu State\,\orcidlink{0000-0002-4751-9577}}}
\affil[1]{SnT, University of Luxembourg, Luxembourg, Luxembourg}
\affil[2]{Constructor University, Bremen, Germany}
\affil[3]{Constructor Knowledge Labs, Bremen, Germany}
\affil[4]{Institute for Functional Intelligent Materials, National University of Singapore, Singapore, Singapore}
\date{}

\begin{document}

\twocolumn[{%
\begin{@twocolumnfalse}
\maketitle
\vspace{-2.4em}
{\centering\small\ttfamily
arkadii.kazanskii@uni.lu,\enspace tatiana.petrova@uni.lu,\enspace Andrey.Ustyuzhanin@constructor.org\\
konstantin.bagrianskii.001@student.uni.lu,\enspace aleksandr.puzikov@uni.lu,\enspace radu.state@uni.lu\par}
\begin{abstract}
Drifting Models (DMs) train a one-step generator by regressing its output
towards a target displaced by a kernel drift. DMF (Drifting Model with Friction)
scales this displacement linearly from full drift to zero. With plain gradient descent, scaling the drift gives the same updates as
linear learning-rate decay. With Adam, it also changes the gradient
moments. For a constant multiplier
$c>0$, its effect in exact arithmetic is to replace Adam's $\epsilon$ by $\epsilon/c$.
In a two-particle model with unnormalised Laplace interactions, we derive the
distance at which attraction balances repulsion. On FFHQ adult-to-child
translation, DMF lowered FID and CMMD relative to fixed-rate DM in all five
seed pairs. In a separate paired comparison, mean FID and CMMD were lower for DM with
linear learning-rate decay than for DMF. Both 95\% intervals for the paired
differences included zero. DMF inherits DM's low training cost. At batch
size 32 on one RTX~A6000, mean training times were 6.85 minutes for DMF,
6.93 for DM and 133.48 for Optimal Flow Matching (OFM). At this batch size,
DMF also achieved lower FID and CMMD than OFM (paired $p=0.002$ and
$0.021$, respectively).

\keywords{Generative models \and Single-step generators \and Domain
translation \and Drifting models \and Learning-rate schedules \and Optimal transport.}
\end{abstract}
\vspace{0.8em}
\end{@twocolumnfalse}
}]

\section{Introduction}
\label{sec:intro}

The number of network evaluations matters when inference time is limited.
Methods that target faster generation through straighter transport
trajectories include Rectified Flow (RF)~\citep{liu2023flow},
optimal-transport conditional flow matching
(OT-CFM)~\citep{tong2024improving}, and Optimal Flow Matching
(OFM)~\citep{kornilov2024optimal}. Under its population assumptions, OFM recovers the straight trajectories of
quadratic optimal transport. During training, OFM solves an inner convex optimisation problem to
evaluate the velocity at sampled interpolation times.

Drifting Models (DMs)~\citep{deng2026drifting} train a generator
$f_\theta$ by regression to a target displaced by a kernel drift.
The population field $V_{p,q}=V_p^+-V_q^-$ combines attraction towards
data and repulsion among generated samples. DMF inherits DM's training
loop and its single generator evaluation at inference. Both methods train
without OFM's inner optimisation and took about seven minutes each in our
matched-batch FFHQ runs (\cref{tab:ffhq}). We ask how decreasing the drift
displacement during training affects translation quality.

\paragraph{Contributions.}
We study a linear schedule for the displacement in the DM regression target.
It adds one scalar multiplication per training step. We derive its
equivalence to learning-rate decay for plain gradient descent and show how
Adam's running moments distinguish the two updates (\cref{sec:discussion}).
On FFHQ, DMF improved both metrics over fixed-rate DM. In the paired
control, DM with linear learning-rate decay had lower mean FID and CMMD,
with both confidence intervals for the differences spanning zero (\cref{sec:ffhq}).
To illustrate particle motion, we construct a scalar model with two
generated particles and unnormalised pairwise interactions. The
contraction threshold we derive applies to this model
(\cref{sec:revisit,sec:dissipation}). The repository contains the base implementation and archived runs.\footnote{\href{https://github.com/BagrKonstantin/Friction-Augmented-Drifting-Models/commit/878a45534accc116c836c1a3782dc138b3117e69}{GitHub commit 878a455}.}

\section{Related Work}
\label{sec:related}

\subsubsection{Straight-trajectory generators.}
Flow matching~\citep{albergo2023stochastic,lipman2023flow} learns a
time-dependent velocity field transporting a source to the data distribution;
the learned trajectories are typically curved, so inference needs multi-step
ODE integration. Rectified Flow~\citep{liu2023flow} straightens trajectories
by iterative reflow; OT-CFM~\citep{tong2024improving} conditions the target
velocity on optimal-transport couplings. Both methods use velocity fields. Straightening the trajectories aims to
reduce the number of integration steps at inference.

\subsubsection{Optimal-transport parametrisations.}
A related line couples generative modelling with
optimal transport (OT)~\citep{korotin2021neural,korotin2023neural}. OFM~\citep{kornilov2024optimal}
parametrises a convex potential with an Input
Convex Neural Network~\citep{amos2017input,makkuva2020optimal}. Its gradient
defines the endpoint transport map and induces a straight-trajectory
velocity field, whose evaluation during training requires a convex inner
optimisation. Under the assumptions of the OFM theory, minimising the
population objective recovers the Brenier map.

\subsubsection{Drifting Models and dissipation.}
Deng et al.~\cite{deng2026drifting} proposed DMs with kernel-weighted
attraction and repulsion. Their Proposition~3.1 establishes
$q=p\Rightarrow V_{p,q}\equiv0$, and Appendix~C.1 discusses a converse
under finite-dimensional and non-degeneracy assumptions.
He et al.~\cite{he2026sinkhorn} use a two-sided
Sinkhorn coupling and analyse its relation to the Sinkhorn divergence.
Our temporal scaling $V_\gamma=(1-\gamma)V$ modifies the training target.
Here \emph{friction} denotes a decreasing force multiplier.
\Cref{app:second-order} considers damping of the velocity instead.
Dissipative terms also occur in Langevin samplers~\citep{welling2011bayesian}
and score-based stochastic differential equations~\citep{song2021scorebased}. For FFHQ
translation we operate in the $W$ latent space of the Adversarial Latent
Autoencoder (ALAE)~\citep{pidhorskyi2020adversarial},
following~Kornilov et al.~\cite{kornilov2024optimal}.

\subsubsection{Learning-rate schedules.}
SGDR reduces the learning rate along a cosine curve between warm
restarts~\citep{loshchilov2017sgdr}. A constant learning rate followed
by a linear cooldown is another studied choice~\citep{schaipp2025cooldown}.
Both schedules act on the optimiser step. Adam normalises gradients using
running moments~\citep{kingma2015adam}, so a constant rescaling largely
cancels when $\epsilon$ is small relative to the gradient scale. Our
paired control compares linear drift decay with linear learning-rate decay
over the same horizon.

\subsubsection{Gaussian identifiability.}
At a fixed Gaussian bandwidth, the population drift is proportional to
the difference of the scores (gradients of the log densities) after
Gaussian smoothing. Global equality of these scores identifies the two
distributions \citep{lai2026unified,turan2026secretly,cao2026gradient}.
This representation follows from the Gaussian posterior-mean identity
\citep{efron2011tweedie,dytso2021derivative} and appears in Weber's
score-difference flow~\citep{weber2023score}. By real analyticity, the
scores need only agree on a nonempty open set (\cref{thm:local-gaussian}).

\section{Methodology}
\label{sec:methodology}

\subsection{Revisiting the Drifting-Field Dynamics}
\label{sec:revisit}

Let $p$ be the target data distribution and $q$ the distribution of the
current generator. The idealised population drift field is
\begin{equation}
  V_{p,q}(x) = V_p^+(x) - V_q^-(x),
  \label{eq:drift}
\end{equation}
where $V_p^+$ points towards the kernel-weighted mean under $p$ and
subtracting $V_q^-$ produces the repulsive contribution:
\begin{align}
  V_p^+(x) &= \tfrac{1}{Z_p}\,\E_{y^+\sim p}\bigl[k(x,y^+)\,(y^+-x)\bigr], \label{eq:vplus}\\
  V_q^-(x) &= \tfrac{1}{Z_q}\,\E_{y^-\sim q}\bigl[k(x,y^-)\,(y^--x)\bigr], \label{eq:vminus}
\end{align}
with $Z_p(x)=\E_{y^+}[k(x,y^+)]$, $Z_q(x)=\E_{y^-}[k(x,y^-)]$ and the Laplace
kernel $k(x,y)=\exp(-\lVert x-y\rVert/\tau)$, $\tau>0$.
\Cref{fig:friction_main} illustrates this attraction--repulsion drift and the
friction scaling $V_\gamma=(1-\gamma)V$ that defines DMF.

\begin{figure}[t!]
  \centering
  \resizebox{0.85\linewidth}{!}{%
  \begin{tikzpicture}

  \definecolor{posbg}{RGB}{210,220,240}
  \definecolor{negbg}{RGB}{245,215,185}
  \definecolor{posdot}{RGB}{80,120,200}
  \definecolor{negdot}{RGB}{230,140,50}
  \definecolor{arrowcol}{RGB}{20,20,20}
  \definecolor{friccolor}{RGB}{0,115,115}

  \begin{scope}
    \clip (-2.9,-2.9) rectangle (2.9,2.9);

    \foreach \s/\op in {2.8/0.10, 2.2/0.14, 1.6/0.18, 1.1/0.22, 0.65/0.26}{
      \fill[negbg,opacity=\op] (-2.4,-1.6) ellipse ({\s*1.5} and {\s*1.2});
    }
    \foreach \s/\op in {2.5/0.10, 1.9/0.14, 1.3/0.18, 0.85/0.22, 0.5/0.26}{
      \fill[posbg,opacity=\op] (2.8, 1.0) ellipse ({\s*1.1} and {\s*1.0});
      \fill[posbg,opacity=\op] (2.2,-0.9) ellipse ({\s*0.9} and {\s*0.85});
    }
    \foreach \x/\y in {
      -1.10/ 0.55, -0.80/ 0.80, -0.45/ 0.92, -0.15/ 0.70,  0.10/ 0.50,
      -1.30/ 0.20, -1.00/ 0.35, -0.60/ 0.45,  0.00/ 0.30, -0.25/ 0.10,
      -1.50/-0.10, -1.15/-0.05, -0.80/ 0.00, -0.40/-0.05,  0.15/-0.10,
      -1.40/-0.40, -1.10/-0.30, -0.70/-0.20, -0.30/-0.30,  0.05/-0.35,
      -1.20/-0.65, -0.90/-0.55, -0.10/-0.55,  0.20/-0.60,
      -0.60/-0.85, -0.25/-0.80,  0.05/-0.80, -0.90/-0.85, -0.40/-1.00,
      -0.70/ 0.65, -0.20/ 1.00,  0.30/ 0.85,  0.35/ 0.20,  0.25/-0.20,
      -1.60/ 0.45, -1.70/-0.25, -1.35/-0.70, -0.05/ 1.10,  0.40/-0.50%
    }{ \fill[negdot] (\x,\y) circle (1.6pt); }
    \foreach \x/\y in {
       0.55/ 0.90,  0.85/ 1.05,  1.15/ 0.85,  1.40/ 0.65,  1.60/ 0.45,
       0.30/ 0.60,  0.65/ 0.65,  1.00/ 0.55,  1.30/ 0.40,  1.55/ 0.15,
       0.10/ 0.30,  0.45/ 0.35,  0.80/ 0.25,  1.10/ 0.15,  1.40/-0.10,
       0.20/ 0.00,  0.55/ 0.05,  0.90/-0.05,  1.20/-0.15,  1.50/-0.35,
       0.05/-0.30,  0.40/-0.25,  0.75/-0.35,  1.05/-0.45,  1.30/-0.60,
       0.20/-0.60,  0.55/-0.55,  0.85/-0.65,  1.10/-0.75,  0.30/-0.85,
       0.65/ 1.20,  0.95/ 1.25,  1.20/ 1.10,  1.50/ 0.90,  1.70/ 0.70,
      -0.05/ 0.80, -0.10/ 0.50,  0.00/-0.05,  0.60/-0.90,  1.70/-0.15%
    }{ \fill[posdot] (\x,\y) circle (1.6pt); }

    \fill[black] (0,0) circle (3.2pt);
    \node[above right, font=\small] at (-0.38,0.08) {$\mathbf{x}$};

    \draw[-{Latex[length=8pt,width=5pt]}, line width=1.5pt, arrowcol, dashed]
      (0,0) -- (-0.55,-0.95) node[below, font=\small] {$\mathbf{V}^-_q$};
    \draw[-{Latex[length=8pt,width=5pt]}, line width=1.5pt, arrowcol]
      (0,0) -- (0.94,-0.34) node[right, font=\small] {$\mathbf{V}^+_p$};
    \draw[-{Latex[length=8pt,width=5pt]}, line width=1.5pt, arrowcol]
      (0,0) -- (1.49, 0.61) node[right, font=\small\bfseries] {$\mathbf{V}$};
    \draw[-{Latex[length=8pt,width=5pt]}, line width=2.5pt, friccolor]
      (0,0) -- (-0.745, -0.305);
    \node[left, friccolor] at (-0.5, -0.45)
      {$\mathbf{V}^f\!=-\!\gamma\mathbf{V}$};
  \end{scope}

  \def\ya{-3.22}
  \def\yb{-3.70}

  \fill[posdot] (-2.8,\ya) circle (3pt);
  \node[right, font=\footnotesize] at (-2.7,\ya) {Data $\mathbf{y}^+\!\sim p$};

  \fill[negdot] (-0.6,\ya) circle (3pt);
  \node[right, font=\footnotesize] at (-0.5,\ya) {Generated samples $\mathbf{y}^-\!\sim q$};

  \draw[-{Latex[length=5pt,width=3.5pt]}, line width=1.3pt, friccolor]
    (-0.87,\yb) -- (-0.55,\yb);
  \node[right, font=\footnotesize, friccolor] at (-0.5,\yb) {Friction $\mathbf{V}^f$};

  \end{tikzpicture}}
  \caption{Drift at a generated point $x$. The vectors $V_p^+$ and
  $V_q^-$ point towards the local weighted means of the data (blue)
  and generated samples (orange), respectively. Their difference
  $V=V_p^+-V_q^-$ is the DM displacement. Adding $V^f=-\gamma V$
  gives the DMF target displacement $(1-\gamma)V$.}
  \label{fig:friction_main}
\end{figure}

\subsubsection{A local unnormalised surrogate.}
To see how attraction and repulsion balance, we analyse a local
\emph{unnormalised pairwise surrogate} inspired by~\eqref{eq:drift}.
It assigns unit weight to one target interaction and one interaction with
the other generated particle, without the population normalisers or
probability weights.
Take one target $y^+=0$ and two generated particles
$x_1^i=+a$, $x_2^i=-a$, with $a>0$. The surrogate drift on $x_1$
sums the raw attraction and repulsion contributions,
\begin{equation}
  \begin{aligned}
  \widehat V(x_1) &= k(x_1^i,y^+)(y^+-x_1^i)\\
  &\quad - k(x_1^i,x_2^i)(x_2^i-x_1^i) \\
  &= -k_t a + 2k_d a,
  \end{aligned}
  \label{eq:two-particle}
\end{equation}
with $k_t=e^{-a/\tau}$ (kernel to target) and $k_d=e^{-2a/\tau}=k_t^2$
(kernel between particles). The normalisers matter even in this geometry.
With $p=\delta_0$ and $q=(\delta_{-a}+\delta_a)/2$, the population
field in~\eqref{eq:drift} includes the self term in $Z_q$ and gives
$V_{p,q}(a)=-a+2a k_d/(1+k_d)<0$ for $a>0$. If the self interaction is
excluded and each remaining interaction is normalised separately, the
two displacements are $-a$ and $-2a$, giving drift $+a$.
Neither convention has a positive force-balance radius in this configuration.
For the unnormalised model~\eqref{eq:two-particle}, the update $x_1^{i+1}=x_1^i+\widehat V(x_1^i)=a(1-k_t+2k_d)$
contracts towards $y^+$ precisely when
\begin{equation}
  k_d < \tfrac{k_t}{2}\ \Longleftrightarrow\ k_t < \tfrac12 \ \Longleftrightarrow\ a > \tau\ln 2.
  \label{eq:stability-no-friction}
\end{equation}
For $a<\tau\ln 2$ the surrogate repulsion overpowers the attraction and the
error $\varepsilon^i=x_1^i-y^+$ obeys
$\varepsilon^{i+1}=\lambda(a^i)\varepsilon^i$ with
$\lambda(a)=1-k_t(a)+2k_d(a)>1$. The nonlinear iteration has a non-zero fixed point at $a=\tau\ln 2$.
Trajectories with positive initial radius approach it from either side.
\Cref{fig:dynamics} illustrates these trajectories and the effect of the
friction schedule. The fixed points are characterised in
\cref{rem:fixed-point} and the finite-horizon bound is given in
\cref{prop:cumulative-bound}.
We call $a<\tau\ln 2$ \emph{locally repulsive} and
reserve ``divergent'' for the linearised behaviour near the target.

\begin{figure*}[t]
  \centering
  \includegraphics[width=\textwidth]{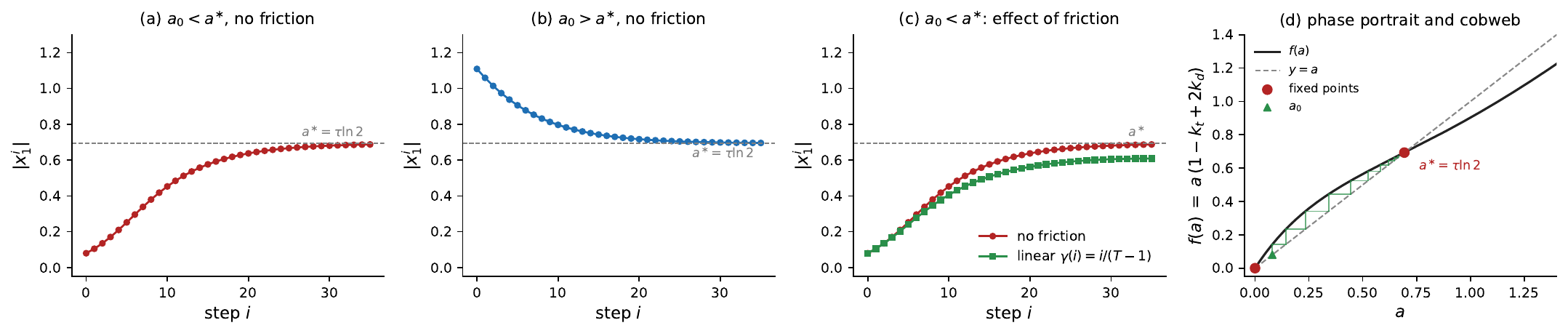}
  \caption{Two-particle surrogate dynamics ($\tau=1$; $a$ is a generated
  particle's distance to the target, so the error magnitude equals $a$).
  (a)~Locally repulsive regime ($a_0<a^{\ast}$) without friction: the error grows
  towards the stable fixed point $a^{\ast}=\tau\ln 2$.
  (b)~Contracting regime ($a_0>a^{\ast}$): the error decays towards $a^{\ast}$.
  (c)~Locally repulsive $a_0$ with/without the linear schedule $\gamma(i)=i/(T-1)$:
  friction halts the trajectory below $a^{\ast}$ as $\gamma\to 1$.
  (d)~Phase portrait of $f(a)=a(1-k_t+2k_d)$ and $y=a$; the two intersections
  are the fixed points of~\cref{rem:fixed-point}.}
  \label{fig:dynamics}
\end{figure*}

\begin{remark}[Fixed-point structure of the surrogate]
\label{rem:fixed-point}
The scalar iteration $a^{i+1}=f(a^i)$, $f(a)=a(1-k_t(a)+2k_d(a))$, has two
fixed points. The trivial $a=0$ is locally expansive ($f'(0)=2$). The
non-trivial $a^{\ast}=\tau\ln 2$ (where $k_t=2k_d$) is locally contracting,
$f'(a^{\ast})=1-\tfrac{\ln 2}{2}\approx 0.653$. From
$f(a)-a=a\,e^{-a/\tau}(2e^{-a/\tau}-1)$, the sign is positive on
$(0,a^{\ast})$ and negative on $(a^{\ast},\infty)$. To see that a step
cannot cross $a^\ast$, put $u=a/\tau$. The bounds $ue^{-u}\le1/e$ and
$|2e^{-u}-1|\le2|u-\ln 2|$ give
$|f(a)-a|\le(2/e)|a-a^\ast|$. Each trajectory from $a_0>0$ is
therefore monotone and bounded between $a_0$ and $a^\ast$. By continuity
its limit is a fixed point, necessarily $a^\ast$.
\end{remark}

\subsection{Drift Scaling in the Two-Particle Model}
\label{sec:dissipation}

To reduce the magnitude of the locally repulsive surrogate update, we
introduce a scalar friction coefficient $\gamma\in[0,1]$ and
apply it to the DM drift,
\begin{equation}
  V_\gamma(x) = (1-\gamma)\,V(x),
  \label{eq:friction-field}
\end{equation}
where $V$ is the full normalised field of~\eqref{eq:drift}. In the surrogate
the same scaling yields
$x_1^{i+1}=a\,\lambda_\gamma(a)$ with
$\lambda_\gamma(a)=1-(1-\gamma)(k_t(a)-2k_d(a))$. Let
$\eta(a)\coloneqq 2k_d(a)-k_t(a)$ be the \emph{signed instability
margin}: $\eta>0$ in the locally repulsive regime and $\le 0$
otherwise. Then $\varepsilon^{i+1}=(1+(1-\gamma)\eta(a^i))\varepsilon^i$,
so at $\gamma=\tfrac12$ the repulsive growth rate is halved relative to the
frictionless case.

Both attraction and repulsion are multiplied by $1-\gamma$.
For $0\le\gamma<1$ the force-balance radius $a^\ast$ is therefore unchanged.
Each scheduled step is a convex combination of $a^i$ and the unscheduled
$f(a^i)$. By \cref{rem:fixed-point}, the interval between $a^0$ and $a^\ast$
is invariant, giving the sharper surrogate bound
$|\varepsilon^i|=a^i\le\max(a^0,a^\ast)$ for every $i$.

\begin{remark}[Finite-horizon product bound]
\label{prop:cumulative-bound}
More generally, let $\eta_i$ be the signed margin at step $i$
($\eta_i=\eta(a^i)$ in the surrogate), with $\max(\eta_i,0)\le\eta_{\max}$.
For any $\gamma_i\in[0,1]$ such that $1+(1-\gamma_i)\eta_i\ge0$,
iteration of the scalar update and $1+u\le e^u$ give
\begin{equation}
\begin{aligned}
 |\varepsilon^T| &\le |\varepsilon^0|\prod_{i=0}^{T-1}
                    (1+(1-\gamma_i)\eta_{\max})\\
 &\le |\varepsilon^0|\exp\!\left(\eta_{\max}\sum_{i=0}^{T-1}(1-\gamma_i)\right).
\end{aligned}
\label{eq:bounded-trajectory}
\end{equation}
The surrogate has $\eta_i\in[-1/8,1]$. For $\gamma_i=i/(T-1)$ and
$T\ge2$, the sum is $T/2$ and the final particle increment is zero.
\end{remark}

\subsection{Practical Implementation}
\label{sec:implementation}

DMF changes one step in the DM training loop (\cref{alg:dmf}). It
multiplies the drift field $V$ by $(1-\gamma(i))$ before forming the
regression target (line~5).
Setting $\gamma(i)=0$ for all $i$ recovers DM exactly. The multiplication
of the batch drift tensor by a scalar needs no extra parameters or forward passes.

For domain translation the generator input is sampled from the source
latent distribution $p_0$; in the 2D task $p_0=\mathcal N(\bm0,\bm I)$.
The implemented batch drift additionally uses batch-scale normalisation,
self-interaction masking, affinities combining row and column softmax
normalisations, and cross-weighted attraction and repulsion. Forces are
normalised separately at each bandwidth and summed before friction is
applied. In \cref{alg:dmf}, $\mathrm{BatchDrift}$ returns this aggregate
$V$ in the original coordinates and the detached input scale $s>0$;
the MSE is evaluated in scaled coordinates. Thus
\eqref{eq:drift}--\eqref{eq:vminus} describe the population idealisation,
not every normalisation in the implementation.

\begin{algorithm}[t]
\caption{DMF training step (one iteration $i$); line~5 is the only addition
over DM, recovered by setting $\gamma\equiv0$.}
\label{alg:dmf}
\footnotesize
\hrule\vspace{2pt}
\noindent\textbf{Require:} generator $f_\theta$; friction schedule
$\gamma(i)\in[0,1]$; source $p_0$; data batch $y^+\sim p$; iteration $i$.
\vspace{2pt}\hrule\vspace{3pt}
\begin{flushleft}
\setlength{\tabcolsep}{4pt}
\begin{tabular}{@{}r@{\quad}l@{}}
1: & $z \sim p_0$ \quad\textit{// source samples}\\
2: & $x \leftarrow f_\theta(z)$ \quad\textit{// generated samples}\\
3: & $y^- \leftarrow x$ \quad\textit{// reuse batch as negatives}\\
4: & $(V,s) \leftarrow \mathrm{BatchDrift}(x;\,y^+,y^-)$\\
5: & $V \leftarrow (1-\gamma(i))\,V$ \quad\textbf{// DMF friction}\ \textit{(omit for DM)}\\
6: & $x_{\mathrm{drift}} \leftarrow \mathrm{stopgrad}(x + V)$\\
7: & $\mathcal{L} \leftarrow \mathrm{MSE}(x/s,\,x_{\mathrm{drift}}/s)$\\
8: & update $\theta$ by descending $\nabla_\theta\mathcal{L}$\\
\end{tabular}
\end{flushleft}
\vspace{1pt}\hrule
\end{algorithm}

\section{Experiments}
\label{sec:experiments}

\subsection{Drift schedules on FFHQ}
We use the FFHQ adult-to-child task, ALAE latents and evaluation metrics
described in \cref{sec:ffhq}.
For each of five seeds, we compared schedules at a batch size of 515 per
domain and trained each model for 30,001 steps
(\nolinkurl{experiments/experiment_2}). The default schedule is
$\gamma(i)=i/(T-1)$, with $i=0,\ldots,T-1$.
\begin{table}[!htb]
  \centering
  \scriptsize
  \caption{Drift-amplitude schedules on FFHQ adult-to-child translation.
  Entries are means and sample standard deviations over five seeds at batch
  size 515 and 30,001 training steps. The baseline uses $\gamma=0$.
  Lower values indicate better quality. Bold marks the lowest observed mean.
  \Cref{tab:ffhq} uses a separate training configuration.}
  \label{tab:schedule}
  \begin{tabularx}{\linewidth}{@{}Xrr@{}}
    \hline
    Schedule $\gamma(i)$, $t=i/(T-1)$ & FID\,$\downarrow$ & CMMD\,$\downarrow$ \\
    \hline
    None ($\gamma\equiv0$)   & $10.91 \pm 0.39$ & $0.0118 \pm 0.0087$ \\
    Delayed-linear, $\max(0,2t-1)$ & $\mathbf{10.49 \pm 0.16}$ & $\mathbf{0.0063 \pm 0.0017}$ \\
    Quadratic, $t^2$ & $10.54 \pm 0.23$ & $0.0066 \pm 0.0023$ \\
    Linear (default), $t$ & $10.58 \pm 0.27$ & $0.0070 \pm 0.0027$ \\
    Constant $0.50$          & $10.62 \pm 0.31$ & $0.0075 \pm 0.0031$ \\
    Square-root, $\sqrt{t}$ & $10.68 \pm 0.19$ & $0.0079 \pm 0.0019$ \\
    Constant $0.75$          & $10.80 \pm 0.53$ & $0.0097 \pm 0.0052$ \\
    Constant $0.25$          & $10.81 \pm 0.35$ & $0.0093 \pm 0.0049$ \\
    Sine, $\sin(\pi t/2)$ & $10.89 \pm 0.12$ & $0.0095 \pm 0.0010$ \\
    \hline
  \end{tabularx}
\end{table}

Relative to full drift, the linear schedule lowered mean FID by 0.328. A
constant $\gamma=0.50$ lowered it by 0.290 (\cref{tab:schedule}). These reductions
are comparable to the baseline's across-seed FID standard deviation of
0.39. Across eight schedules and two metrics, none of the 16 paired
comparisons with full drift reached 0.05 after Holm adjustment (smallest
adjusted $p=0.123$). For seed~42 and constant $\gamma=0.50$, re-evaluating
the archived weights with shared decoder noise preserved DMF's advantage
on both metrics. After retraining both methods on Apple Silicon, the same
evaluation favoured DM on both.

\subsection{Gaussian Mixtures in 2D}
\label{sec:toy}

We construct a 2D problem whose target is a Gaussian mixture
$p_1=\sum_{k=1}^{2}w_k\,\mathcal{N}(\bm{\mu}_k,\bm{\Sigma}_k)$ with
$\bm{\mu}_1=(-2,0)$, $\bm{\mu}_2=(2,0)$,
$\bm{\Sigma}_1=\bm{\Sigma}_2=\mathrm{diag}(1,1)$, $w_1=w_2=0.5$, and source
$p_0=\mathcal{N}(\bm{0},\bm{I})$. The generator is a residual MLP with two
hidden layers of width $1024$. The drift uses kernel bandwidths
$R\in\{0.02,0.05,0.15\}$. We compare means and covariances using the 2D
Fr\'echet distance (FD) on $100{,}000$ generated samples. On
$5{,}000$-sample subsets we compute empirical 2-Wasserstein distance
($W_2$) by optimal transport and the $W_2^2$-based L2 unexplained-variance
percentage (L2-UVP). Across three seeds, mean FD was 0.00017 for DMF and
0.0024 for DM, a roughly fourteenfold reduction in first- and
second-moment discrepancy. The empirical transport metrics were much closer. For DMF and DM,
respectively, mean $W_2$ was 0.183 and 0.188, and mean L2-UVP was 0.56\%
and 0.60\% (\cref{tab:toy}).

\begin{table*}[t]
  \centering
  \small
  \caption[Gaussian-to-mixture translation in 2D]{2D Gaussian~$\to$~Gaussian-mixture toy task (mean\,$\pm$\,std over
  3 seeds; batch $1024$, $500$ iterations). FD is the Fr\'echet distance (the
  mean/covariance metric underlying FID), $W_2$ the empirical 2-Wasserstein
  distance (via optimal transport between sample sets), and L2-UVP the $W_2^2$-based
  unexplained-variance percentage. DM and DMF share an identical architecture,
  optimiser, and learning rate ($10^{-4}$); only the friction coefficient
  differs. Only three-seed summaries are available for this mixture.
  Bold marks the smallest mean FD.}
  \label{tab:toy}
  \begin{tabular}{@{}lrrr@{}}
    \hline
    Method & FD\,$\downarrow$ & $W_2$\,$\downarrow$ & L2-UVP\,\%\,$\downarrow$ \\
    \hline
    OFM~\citep{kornilov2024optimal}         & $0.0043 \pm 0.0027$ & $0.248 \pm 0.125$ & $1.29 \pm 1.21$ \\
    DM~\citep{deng2026drifting}             & $0.0024 \pm 0.0009$ & $0.188 \pm 0.021$ & $0.60 \pm 0.13$ \\
    DMF ($\gamma=0\!\to\!1$, \textbf{ours}) & $\mathbf{0.00017 \pm 0.00007}$ & $0.183 \pm 0.018$ & $0.56 \pm 0.11$ \\
    \hline
  \end{tabular}
\end{table*}

\subsubsection{Kernel comparison.}
To test sensitivity to the choice of kernel, we use a second mixture with
component means $(0,0)$ and $(3,2)$ and a standard Gaussian source.
At batch size 16 and 10,000 steps, mean FD over seeds 42, 123 and 777
fell from 0.09009 with full drift to 0.00420 with linear drift decay for
the Laplace kernel, and from 0.10269 to 0.00506 for the Gaussian kernel.
The Laplace values come from \nolinkurl{experiments/experiment_4} on this
second mixture. The Gaussian values come from \nolinkurl{experiments/experiment_8}.

\subsection{Domain Translation on FFHQ}
\label{sec:ffhq}

Following~Kornilov et al.~\cite{kornilov2024optimal}, we map samples between
attribute-conditioned subsets of the $W$ latent space of
ALAE~\citep{pidhorskyi2020adversarial}, trained on
FFHQ~\citep{karras2019style} at $1024\!\times\!1024$. Both DM and DMF
use the Laplace kernel for these experiments. We evaluated translation quality using FID~\citep{heusel2017fid} and CMMD~\citep{jayasumana2024cmmd}.
FID compares the first two moments of Inception features. The CMMD
implementation uses a Gaussian kernel on CLIP image embeddings and
includes the diagonal terms in its within-sample averages.

\begin{table*}[!t]
  \centering
  \small
  \caption{Quality and training time on FFHQ adult-to-child translation.
  Means and sample standard deviations use five seeds for DM, DMF and
  each OFM row. With OFM at batch 128, paired comparisons with DMF give
  $p=0.367$ for FID and $p=0.168$ for CMMD. Equivalence was not tested.
  Times are minutes per run on one RTX~A6000. RF, 2-RF and OT-CFM have
  three seeds each and no recorded times. Bold marks the lowest mean FID and CMMD.}
  \label{tab:ffhq}
  \begin{tabular}{@{}lrrr@{}}
    \hline
    Method & FID\,$\downarrow$ & CMMD\,$\downarrow$ & Time, min \\
    \hline
    OFM, batch 32~\citep{kornilov2024optimal} & $12.642 \pm 0.745$ & $0.03378 \pm 0.01614$ & $133.48 \pm 27.20$ \\
    OFM, batch 128                    & $10.572 \pm 0.299$ & $0.01175 \pm 0.00473$ & $197.54 \pm 14.27$ \\
    RF~\citep{liu2023flow}             & $15.214 \pm 1.898$ & $0.04423 \pm 0.01356$ & -- \\
    2-RF~\citep{liu2023flow}           & $16.375 \pm 1.961$ & $0.06060 \pm 0.01995$ & -- \\
    OT-CFM~\citep{tong2024improving}   & $10.810 \pm 0.450$ & $0.01450 \pm 0.00224$ & -- \\
    \hline
    DM, batch 32~\citep{deng2026drifting} & $11.195 \pm 0.439$ & $0.01547 \pm 0.00284$ & $6.93 \pm 0.62$ \\
    DMF, batch 32                     & $\mathbf{10.413 \pm 0.130}$ & $\mathbf{0.00746 \pm 0.00219}$ & $6.85 \pm 0.21$ \\
    \hline
  \end{tabular}
\end{table*}

\subsubsection{Training setup.}
DM, DMF and OFM use the same ALAE latent representation and train for
10,001 steps in the five-seed comparison (\cref{tab:ffhq}). DM and
DMF use Adam at a fixed learning rate of $5\times10^{-4}$ and batch
size~32. OFM uses Adam at $10^{-3}$ with an inner convex optimisation
at each step. We report its batch-32 and batch-128 configurations separately.
Times are for one NVIDIA RTX~A6000. FID and CMMD use the relevant
attribute subset of the held-out test split (the last 10,000 of 70,000
FFHQ ALAE latents). We pair runs by seed (42, 123, 777, 1234 and 2024)
and report two-sided paired $t$-tests and standardised paired differences
$d_z$. Per-run configurations and metrics for DM, DMF and both OFM
configurations are in \nolinkurl{experiments/experiment_1}. We report our RF,
2-RF and OT-CFM runs on the same latents as three-seed summaries.

\subsubsection{Fixed learning rate.}
Both FID and CMMD were lower for
DMF than for fixed-rate DM in all five pairs
($p=0.019$, $d_z=1.71$ for FID and $p=0.005$, $d_z=2.55$ for CMMD).
The mean paired reductions (DM minus DMF) were 0.782 in FID
(95\% paired-$t$ interval 0.214 to 1.349) and 0.00801 in CMMD
(0.00411 to 0.01191). Mean training time was 6.93 minutes for DM
and 6.85 for DMF.
Both DM and DMF use a single generator pass at inference. At the matched batch size of 32,
OFM took 133.48 minutes on average (19.5 times as long as DMF) and had
higher FID and CMMD in all five pairs (paired $p=0.002$ and $0.021$,
respectively). At batch size 128, OFM took 197.54 minutes (28.8 times
as long as DMF), with mean scores closer to DMF's. OFM also permits direct
evaluation of its learned transport map at inference~\citep{kornilov2024optimal}.
\Cref{fig:ffhq} shows translations for six subjects.

\begin{figure*}[!t]
  \centering
  \begin{minipage}{0.62\textwidth}
  \centering
  \makebox[\linewidth]{\hfill\makebox[0.24\linewidth]{\bfseries Input}\hfill
  \makebox[0.24\linewidth]{\bfseries DM}\hfill
  \makebox[0.24\linewidth]{\bfseries DMF}\hfill
  \makebox[0.24\linewidth]{\bfseries OFM}\hfill}\par\smallskip
  \includegraphics[width=\linewidth,trim=0 0 0 51bp,clip]{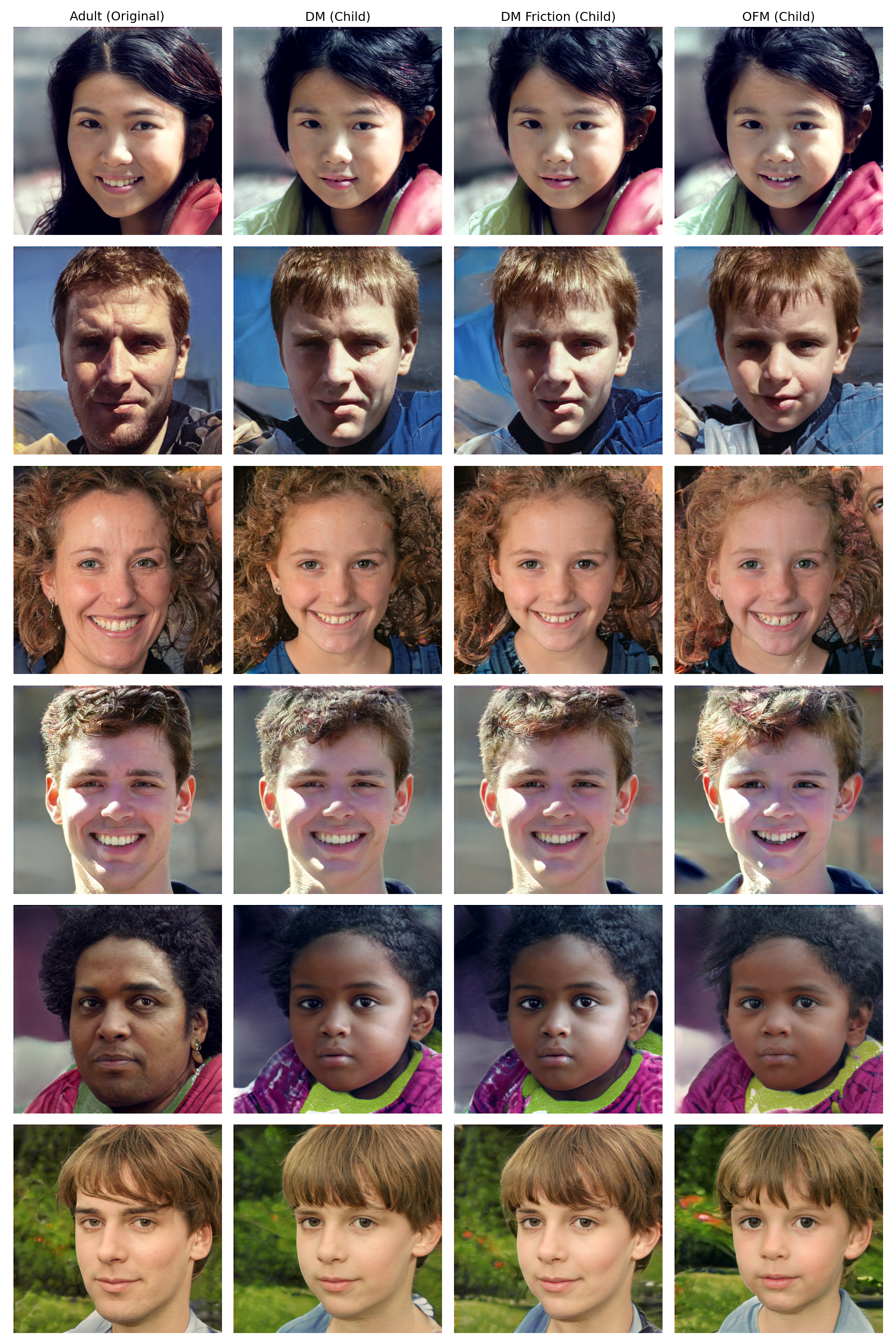}
  \end{minipage}
  \caption{Qualitative FFHQ adult~$\to$~child translation; each row shows one of six subjects. Left to right: original adult image, DM, DMF
  ($\gamma(i)=i/(T-1)$), OFM. These six examples show no consistent visible difference between DM and DMF. FID and CMMD quantify differences across distributions.}
  \label{fig:ffhq}
\end{figure*}

\subsubsection{Paired learning-rate control.}
To compare drift scaling with learning-rate decay, we trained DM and
DMF on an Apple M3 Pro (MPS, float32) using the five seeds from
\cref{tab:ffhq}. At step $i=0,\ldots,10{,}000$,
let $c_i=1-i/10{,}000$. DM used the full drift and Adam learning rate
$5\times10^{-4}c_i$. DMF used drift $c_i V$ and a fixed learning rate
of $5\times10^{-4}$. Each pair started from identical weights and zero
Adam moments and used the same 32 adult and 32 child training examples
at each step. We evaluated the final weights on all 8,104 adults and
1,819 children of known age in the test split. We used the same decoder noise and the same decoded child reference set to
evaluate both methods. This batch-32 MPS
control has no fixed-rate, full-drift DM arm. Hardware and evaluation differ
from \cref{tab:ffhq}, so only within-experiment differences are paired.
Full-precision scores for the five pairs are in the arXiv ancillary file
\nolinkurl{anc/ffhq_lr_pairs.csv}.

\begin{table}[!ht]
  \centering
  \caption{Paired FFHQ adult-to-child comparison on Apple M3 Pro.
  DM-LR decays the learning rate linearly. DMF decays the drift multiplier
  linearly at fixed learning rate. Lower values are better.}
  \label{tab:ffhq-lr}
  \small
  \begin{tabular}{@{}r@{\quad}rr@{\quad}rr@{}}
    \hline
    & \multicolumn{2}{c}{FID} & \multicolumn{2}{c}{CMMD} \\
    Seed & DM-LR & DMF & DM-LR & DMF \\
    \hline
    42   & 10.131 & 10.247 & 0.00668 & 0.00727 \\
    123  & 10.062 & 10.118 & 0.00739 & 0.00715 \\
    777  & 10.120 & 10.132 & 0.00536 & 0.00644 \\
    1234 & 10.261 & 10.671 & 0.00727 & 0.01180 \\
    2024 & 10.047 & 10.070 & 0.00668 & 0.00572 \\
    \hline
    Mean & 10.124 & 10.248 & 0.00668 & 0.00768 \\
    \hline
  \end{tabular}
\end{table}

DM with learning-rate decay had lower FID in five of five pairs and
lower CMMD in three (\cref{tab:ffhq-lr}). The mean paired difference
(DM-LR minus DMF) was $-0.123$ for FID (95\% paired-$t$ interval
$-0.329$ to $0.082$, $p=0.171$) and $-0.00100$ for CMMD
($-0.00363$ to $0.00163$, $p=0.350$). The observed means favoured DM
with learning-rate decay, and both intervals included zero.

\subsubsection{Other FFHQ settings.}
For male-to-female translation (three seeds), mean FID was $6.94\pm0.01$
for DMF, $6.98\pm0.15$ for DM and $7.62\pm0.03$ for OFM. Paired
comparisons of DMF with DM and OFM gave $p=0.69$ and $p<10^{-3}$,
respectively. Mean CMMD was $0.0043\pm0.0007$ for DMF, $0.0055\pm0.0012$
for DM and $0.0111\pm0.0036$ for OFM ($p=0.36$ and $p=0.06$ for the
same comparisons). The per-run metrics are in \nolinkurl{experiments/experiment_3}.

In the limited-data experiment (\nolinkurl{experiments/experiment_10}),
each run sampled half the eligible latents in each age group from the
first 60,000 FFHQ entries. Each training subset contained 24,393 adults
and 5,381 children. We used three seeds, batch size 32 and 10,001 steps. DMF
has lower mean FID ($10.37\pm0.24$ versus $11.27\pm0.73$ for DM)
and CMMD ($0.0083\pm0.0015$ versus $0.0182\pm0.0047$) in this subset.

\section{Discussion}
\label{sec:discussion}

On FFHQ adult-to-child translation, DMF improved both FID and CMMD over
DM trained with a fixed Adam learning rate (\cref{tab:ffhq}). In the
paired control, DM with linear learning-rate decay had lower mean FID and
CMMD, and both 95\% intervals for the paired differences included zero
(\cref{tab:ffhq-lr}).

\subsubsection{Drift scaling and the optimiser.}
In the unnormalised two-particle model, the unscheduled trajectory
from $0<a^0<a^\ast$ approaches $a^\ast=\tau\ln 2$. With the linear schedule, the last increment is zero and the trajectory
ends at radius $a^T<a^\ast$ (\cref{fig:dynamics}c). The network instead takes parameter
updates computed from the regression loss in \cref{alg:dmf}.

Write $c_i=1-\gamma(i)$ and let $g_i$ be the gradient of the unscaled
loss at the same parameters and on the same batch. Since the target
and batch scale are detached, scaling the target displacement gives
the instantaneous gradient $c_i g_i$. For plain SGD without momentum or additional terms, this gives the same
update as using a learning rate of $c_i\alpha_i$.
Adam~\cite{kingma2015adam} accumulates first and second moments of
$c_i g_i$ and $c_i^2 g_i^{\odot 2}$, respectively. When $c_i$ varies,
the history can change the direction of an update. At fixed moments,
changing the learning rate simply rescales the current update.
Stored momentum can also give a nonzero final parameter step at $c_i=0$.

For constant $c>0$, zero-initialised moments, and the same underlying
gradient sequence, the bias-corrected moments scale by $c$ and $c^2$.
Coordinatewise, Adam's normalised update direction becomes
\[
 \frac{c\widehat m_i}{c\sqrt{\widehat v_i}+\epsilon}
 =\frac{\widehat m_i}{\sqrt{\widehat v_i}+\epsilon/c}.
\]
The scale cancels in exact arithmetic at $\epsilon=0$ wherever the
denominator is positive. At positive $\epsilon$ the effective offset
changes.

\subsubsection{Constant coefficients in the first Adam updates.}
We replayed the first four updates of the batch-515 configuration in
\cref{tab:schedule} on CPU, in float32 and float64. For each of the five
seeds and each $c\in\{0.75,0.50,0.25\}$, the scaled run and its two
full-drift controls shared initial weights and batches. One control used
the same $\epsilon=10^{-8}$ as the scaled run. The other used $\epsilon/c$.
We measured the relative difference between parameter updates as
$\|\Delta\theta_c-\Delta\theta_{\rm control}\|_2 /
\|\Delta\theta_{\rm control}\|_2$.

At the common initial state, the float64 gradient agreed with $c g$
to relative error below $9\times10^{-16}$ in all 15 cases. With the
same $\epsilon$, the first parameter updates nevertheless differed by
0.204--1.505\%. With $\epsilon/c$ in the control, the relative difference
stayed below $4.6\times10^{-14}$ through all four updates in float64.
In float32 the largest difference with this adjustment was
$2.5\times10^{-5}$ on step one and $6.1\times10^{-3}$ on step four.
Thus the effective $\epsilon$ accounts for the initial discrepancy in
double precision. It remains unclear how much the initial discrepancy or subsequent rounding
affects FID and CMMD after 30,001 steps.

\subsubsection{Particle dynamics and population equilibrium.}
In the two-particle Laplace model, a positive drift multiplier scales the
increment at each radius and preserves the force-balance radius
(\cref{sec:dissipation}). For the Gaussian population
field, exact vanishing on a nonempty open set at one fixed bandwidth
identifies $p=q$ (\cref{thm:local-gaussian}). The linear schedule sets the
final drift multiplier to zero for every generator distribution.

\subsubsection{Limitations and future work.}
The paired learning-rate control covers five seeds, one linear schedule
and one domain pair.
Measuring attraction and repulsion during training would test whether the
two-particle mechanism describes the implemented batch drift. Further directions include a Gaussian-kernel FFHQ study,
quantitative convergence rates, state-dependent schedules and the
momentum-based variant of \cref{app:second-order}.

\subsubsection{Broader impact.}
Face attribute translation can be used to alter a person's apparent age
or identity without consent. The FFHQ experiments measure agreement
between image distributions. Individual identity preservation requires
a separate evaluation. The code is intended for reproducing the
experiments on public benchmarks.

\section{Conclusion}
On FFHQ adult-to-child translation, linear drift annealing improved FID
and CMMD over fixed-rate DM in five seed pairs at comparable training
cost. In the matched comparison, DM with linear learning-rate decay had lower
mean FID and CMMD. Both intervals for the paired differences included
zero.
The comparison supports drift annealing as a scheduling choice and
leaves its advantage over learning-rate decay unresolved.

\subsubsection*{AI Usage.}
We used Claude Code to assist with manuscript editing.

\appendix
\section{Local Identifiability of the Gaussian Drift}
\label{app:identifiability}

\subsection{Global and local identifiability}
Deng et al.~\cite{deng2026drifting} (App.~C.1) note that zero drift alone
does not identify the distributions for an arbitrary field. For their
kernelised field they infer $p=q$ from identically zero drift, assuming a
finite-dimensional model of full-support densities and linear independence
of the pairwise interaction vectors at the test locations.

For the normalised Gaussian drift, vanishing of the population field
throughout $\R^d$ at one fixed bandwidth implies $p=q$. Turan and
Ovsjanikov~\cite{turan2026secretly} give the score identity in Theorem~4.1.
Their Proposition~5.1 assumes vanishing at all positive bandwidths.
The proof in Appendix~A.2 uses only one fixed bandwidth. Lai et al.~\cite{lai2026unified} identify the Gaussian
drift with the smoothed-score difference in Theorem~1 and recover the
measures from globally equal scores in Proposition~1. Cao et al.~\cite{cao2026gradient}
prove global identifiability for a family of gradient flows under their
kernel and convexity assumptions (Thm.~4.7). For $f(u)=u\log u$ and a
Gaussian kernel, their velocity agrees with the drifting field up to
a positive constant (Thm.~4.9 and Cor.~4.10), giving the same global
criterion.

The Gaussian case gives a population criterion for interpreting a zero
drift. The FFHQ experiments use a Laplace kernel and are outside its scope.
Cheng~\cite{cheng1995meanshift} gives the Gaussian weighted-mean
identity (Thm.~2, Eq.~21), and Weber~\cite{weber2023score} gives its two-score
form (Eq.~14).

\Cref{thm:local-gaussian} is a local consequence of the global
criterion. Its hypothesis requires an exact zero of the unscaled field
at every point of a nonempty open set. Real analyticity then gives
equality of the scores on all of $\R^d$. Normalisation and injectivity
of Gaussian convolution then recover the measures
\citep{turan2026secretly,zhang2020spread,sriperumbudur2010hilbert}.

\subsection{Statement and proof}
For $d\ge1$ and a fixed $\tau>0$, set
\begin{align}
 \psi_\tau(u)&=\exp\bigl(-\lVert u\rVert^2/(2\tau^2)\bigr),\\
 Z_\mu(x)&=\int_{\R^d}\psi_\tau(x-y)\,d\mu(y),\label{eq:gaussian-Z}\\
 V_\mu^{\mathrm G}(x)&=\frac{\int_{\R^d}(y-x)\psi_\tau(x-y)\,d\mu(y)}{Z_\mu(x)}.
 \label{eq:gaussian-V}
\end{align}
Here $\mu$ is a Borel probability measure on $\R^d$ and
$V_{p,q}^{\mathrm G}=V_p^{\mathrm G}-V_q^{\mathrm G}$ is the Gaussian
instance of \eqref{eq:drift}. The superscript distinguishes it from the
Laplace field used in the main text.

\begin{theorem}[Local Gaussian identifiability]
\label{thm:local-gaussian}
Let $p,q$ be Borel probability measures on $\R^d$, with $d\ge1$, and let
$\tau>0$ be fixed. If there exists a nonempty open set $U\subseteq\R^d$
such that $V_{p,q}^{\mathrm G}(x)=0$ for every $x\in U$, then $p=q$.
\end{theorem}

\begin{proof}
\textbf{1. Well-posedness.}
Since $0<\psi_\tau\le1$, we have $0<Z_\mu(x)\le1$. The numerator in
\eqref{eq:gaussian-V} is absolutely integrable because
\[
 \sup_{r\ge0}r e^{-r^2/(2\tau^2)}=\tau/\sqrt e.
\]
\textbf{2. Gaussian score identity.}
Differentiating the Gaussian kernel gives the weighted-mean identity
\citep{efron2011tweedie,dytso2021derivative,cheng1995meanshift} for
the measures considered here. We have
\[
 \nabla_x\psi_\tau(x-y)=\tau^{-2}(y-x)\psi_\tau(x-y),
\]
whose norm is at most $1/(\tau\sqrt e)$, uniformly in $x,y$.
Dominated differentiation therefore gives
\begin{equation}
 V_\mu^{\mathrm G}=\tau^2\nabla\log Z_\mu,
 \qquad
 V_{p,q}^{\mathrm G}=\tau^2\nabla\log(Z_p/Z_q).
 \label{eq:gaussian-score}
\end{equation}

The difference is the score representation in Weber~\cite{weber2023score}
(Eq.~14) and Turan and Ovsjanikov~\cite{turan2026secretly} (Thm.~4.1).

\textbf{3. Analyticity of Gaussian smoothing.}
For $z\in\mathbb C^d$, define
\[
 \mathcal Z_\mu(z)=\int_{\R^d}
 \exp\!\left(-\frac{\sum_{j=1}^d(z_j-y_j)^2}{2\tau^2}\right)d\mu(y).
\]
The sum uses complex squares, not the Hermitian norm. Writing
$z=x+\mathrm i\xi$, the modulus of the integrand is
\[
 e^{-\lVert x-y\rVert^2/(2\tau^2)}
 e^{\lVert\xi\rVert^2/(2\tau^2)}.
\]
On each compact subset of $\mathbb C^d$ this is bounded by a finite
constant independent of $y$, hence by a $\mu$-integrable majorant.
Dominated convergence gives continuity. With the other coordinates fixed,
Fubini permits integration around a triangle inside the $\mu$-integral.
The inner contour integral vanishes since the integrand is entire.
Morera's theorem gives holomorphy in each coordinate. Hartogs' theorem
on separate analyticity~\citep{hartogs1906theorie} then gives joint
holomorphy. Thus $\mathcal Z_\mu$ is entire and its restriction
$Z_\mu$ to $\R^d$ is real-analytic.

\textbf{4. From local equality to recovery of the measures.}
Put $h=\log Z_p-\log Z_q$. Positivity and Step~3 make $h$ and each
partial derivative $\partial_j h$ real-analytic on $\R^d$. By
\eqref{eq:gaussian-score}, every $\partial_j h$ vanishes on $U$.
The identity theorem for real-analytic functions
\citep{krantz2002primer} gives $\partial_j h=0$ throughout the connected
space $\R^d$, for every $j$. Thus $h$ is constant and $Z_p=cZ_q$
globally for some $c>0$. This argument works for disconnected $U$ as well.

From this global proportionality, the normalisation and Fourier steps
coincide with those in Turan and Ovsjanikov~\cite{turan2026secretly}
(App.~A.2).
As in Hyv\"arinen's equal-score argument~\cite{hyvarinen2005score}
(App.~B), integration fixes the constant. Here Tonelli's theorem gives
\[
 \int_{\R^d}Z_\mu(x)\,dx
 =\int_{\R^d}\psi_\tau(u)\,du=(2\pi\tau^2)^{d/2},
\]
so $c=1$ and $p*\psi_\tau=q*\psi_\tau$.

Use the Fourier convention
$\widehat\mu(\omega)=\int e^{-\mathrm i\omega\cdot y}\,d\mu(y)$,
and the same convention for integrable functions. The preceding
normalisation shows that $Z_p,Z_q\in L^1(\R^d)$; the convolution identity
for a finite measure and an $L^1$ function yields
\[
 \widehat p(\omega)\widehat\psi_\tau(\omega)
 =\widehat q(\omega)\widehat\psi_\tau(\omega).
\]
Since
\[
 \widehat\psi_\tau(\omega)
 =(2\pi\tau^2)^{d/2}e^{-\tau^2\lVert\omega\rVert^2/2}>0,
\]
we obtain $\widehat p=\widehat q$. Uniqueness of characteristic functions
gives $p=q$ (J\o rgensen and Labouriau~\cite{jorgensen2012exponential},
App.~A.4, Thm.~A.10, p.~183). Equivalently, Rudin's uniqueness theorem
for the inverse transform (\cite{rudin1962fourier}, \S1.3.6, p.~17)
applies to $p-q$ at $-\omega$.
\end{proof}

\subsection{Interpretation and limits}
The pointwise Gaussian posterior mean is
\[
 m_\mu(x)=\frac{\int_{\R^d}y\psi_\tau(x-y)\,d\mu(y)}{Z_\mu(x)}
          =x+V_\mu^{\mathrm G}(x).
\]
Its numerator is absolutely integrable for each fixed $x$, even when
$\mu$ has no first moment: the triangle inequality gives
$\lVert y\rVert\psi_\tau(x-y)\le\lVert x\rVert+\tau/\sqrt e$.
The displayed ratio is the mean of the posterior probability measure
proportional to
$\psi_\tau(x-y)d\mu(y)$ for independent additive $N(0,\tau^2 I)$ noise.
Thus $V_{p,q}^{\mathrm G}=m_p-m_q$. Dytso et al.~\cite{dytso2021derivative}
give global identification of a prior from its posterior mean in the
vector case (Cor.~3) and trace the scalar result to Wu and
Verd\'u~\cite{wu2012functional}. In dimension one, Lemma~2 of Dytso
et al.\ gives real analyticity of the posterior mean. Their lemma and
Corollary~3, together with the identity theorem, already yield
identification from equality on an open interval. Steps~3--4 above
give the open-set argument directly in every dimension.

\begin{remark}[Population condition and scheduled dynamics]
The theorem allows empirical probability measures. Its hypothesis concerns
where the field is evaluated. He et al.~\cite{he2026sinkhorn} (Lemma~A.1)
construct distinct two-point measures with zero Gaussian drift at both
points in the support of $p$. Testing that finite set therefore cannot
replace exact vanishing on an open ball.
For a fixed $0\le\gamma<1$, multiplication by $1-\gamma$ preserves the
zeros of the field; at $\gamma=1$, the scaled field vanishes for every
$p,q$. The theorem identifies distributions from exact population zeros.
Convergence of the scheduled algorithm and error bounds from a small
empirical drift require further assumptions.
\end{remark}

\begin{remark}[Kernel scope]
For the Laplace kernel $\exp(-\lVert u\rVert/\tau)$, the derivative contains
$u/\lVert u\rVert$ for $u\ne0$, so the Gaussian score identity
\eqref{eq:gaussian-score} does not hold for the drift defined by
\eqref{eq:vplus}--\eqref{eq:vminus}. For radial kernels satisfying their
hypotheses, Lai et al.~\cite{lai2026unified} (Thm.~2) decompose mean shift
into a scalar multiple of the kernel-smoothed score and a covariance term.
They specialise this decomposition to the Laplace kernel.
Analyticity also fails in general. Convolving the Laplace kernel with
$\delta_0$ returns the kernel itself, which is not differentiable at the
origin. Cao et al.~\cite{cao2026gradient} (App.~J, Table~3) record the
resulting failure of their $C^1$ condition K4. The proof of
\cref{thm:local-gaussian} therefore does not establish the corresponding
local criterion for the Laplace field or for mixtures of bandwidths.
\end{remark}

\raggedend
\raggedbottom
\newpage
\section{Velocity damping in the two-particle model}
\label{app:second-order}
For comparison, consider a velocity update with $v^0=0$ and
$\beta_i=1-\gamma_i$:
\[
 v^{i+1}=\beta_i v^i+V(x^i),\qquad x^{i+1}=x^i+v^{i+1}.
\]
Since $v^i=x^i-x^{i-1}$ for $i\ge1$, this is the heavy-ball form with
time-dependent momentum~\citep{polyak1964heavy}. At $\gamma_i=1$ it
retains the force increment $V(x^i)$.

Freeze the scalar margin at $\eta$ so that $V(x)=\eta x$.
The matrix acting on $(x,v)$ and its characteristic polynomial are
\[
 M=\begin{pmatrix}1+\eta&\beta\\ \eta&\beta\end{pmatrix},\qquad
 P(\lambda)=\lambda^2-(1+\eta+\beta)\lambda+\beta .
\]
For $0\le\beta\le1$, both eigenvalues have modulus below one precisely
when $\beta<1$ and $-2(1+\beta)<\eta<0$ by the quadratic Schur
criterion~\citep{elaydi2005introduction}. At $\beta=1$ their product is
one. For $\eta=0$ the eigenvalues are $1,\beta$. For $\eta>0$,
$P(1)<0$ gives a root greater than one.

For $\eta_i\in[0,\eta_{\max}]$ and $\beta_i=1-i/(T-1)$, $T\ge2$,
the row-sum norm of each update matrix is at most $1+\eta_{\max}+\beta_i$.
Multiplying the bounds and using $\sum_i\beta_i=T/2$ yields
\[
 \max(|x^T|,|v^T|)\le |x^0|\,e^{\eta_{\max}T+T/2}.
\]
This upper bound exceeds the first-order bound in
\cref{prop:cumulative-bound} when $|x^0|>0$.

\end{document}